\documentclass{llncs}
\usepackage{times}

\usepackage{amsmath,amsfonts,amssymb}

\usepackage{graphicx}
\usepackage{algorithmic}
\usepackage{algorithm}
\usepackage{enumerate}

\usepackage{tikz}
\usetikzlibrary{arrows,automata,positioning}

\usepackage{tabu}
\usepackage{booktabs}
\usepackage{arydshln}
\newtheorem{observation}{Observation}

\usepackage{url}

\newcommand{\ie}{i.e.}

\newcommand{\eg}{e.g.}

\newcommand{\virg}[1]{``#1''}

\newcommand{\tuple}[1]{\ensuremath{\langle #1 \rangle}}
\newcommand{\set}[1]{\ensuremath{\{#1\}}}
\newcommand{\powset}[1]{\ensuremath{2^{#1}}}
\newcommand{\LAND}{\bigwedge}
\newcommand{\LOR}{\bigvee}

\newcommand{\winner}{\ensuremath{\mbox{\textbf{Cegartix}}}}
\newcommand{\maxsat}{\ensuremath{\mbox{\textbf{prefMaxSAT}}}}
\newcommand{\asp}{\ensuremath{\mbox{\textbf{prefASP}}}}

\newcommand{\attackers}[1]{#1^{-}}

\newcommand{\MAXSATSOLVER}{\ensuremath{\mathit{MXSS}}}
\newcommand{\satproblem}{\ensuremath{\Pi}}
\newcommand{\ARGSTRUE}{\ensuremath{\alpha}}
\newcommand{\MODELS}{\ensuremath{\eta}}

\newcommand{\K}{\ensuremath{\mathcal{K}}}
\newcommand{\sol}[1]{\ensuremath{\sigma(#1)}}

\newcommand{\argument}[1]{\ensuremath{\textrm{\textbf{#1}}}}
\newcommand{\arga}{\argument{a}}
\newcommand{\argb}{\argument{b}}
\newcommand{\argc}{\argument{c}}

\newcommand{\AFname}{\ensuremath{AF}}
\newcommand{\setargs}{\ensuremath{\mathcal{A}}}
\newcommand{\setattacks}{\ensuremath{\mathcal{R}}}
\newcommand{\AF}[1]{\tuple{\setargs_{#1}, \setattacks_{#1}}}
\newcommand{\anAF}{\AF{}}
\newcommand{\anAFsymbol}{\ensuremath{\Gamma}}
\newcommand{\AFsymbol}[1]{\ensuremath{\Gamma_{#1}}}
\newcommand{\attacks}[2]{\ensuremath{#1 \rightarrow #2}}

\newcommand{\carfun}[1]{\ensuremath{\mathrm{F}_{#1}}}                 

\newcommand{\dungconffree}{conflict--free}
\newcommand{\dungacceptable}{acceptable}
\newcommand{\dungcharacteristic}{characteristic}
\newcommand{\dungadmissible}{admissible}

\newcommand{\dunggrounded}{grounded}
\newcommand{\dungpreferred}{preferred}

\newcommand{\aset}{\ensuremath{S}}

\newcommand{\attacked}[1]{#1^{+}}             

\newcommand{\gensem}{\mathsf{S}}
\newcommand{\setgenext}[2]{\mathcal{E}_{#1}(#2)}

\newcommand{\GR}{\mathcal{GR}}
\newcommand{\PR}{\mathsf{PR}}
\newcommand{\setadmwin}[1]{\mathcal{E}_{\mathit{adm}}^{iun}(#1)}   

\newcommand{\proplang}{\ensuremath{\mathcal{L}}}
\newcommand{\vfunsym}{\ensuremath{v}}
\newcommand{\vfun}[1]{\ensuremath{\vfunsym(#1)}}
\newcommand{\admenc}[1]{\ensuremath{\mathit{admsat}_{#1}}}
\newcommand{\admaspenc}[1]{\ensuremath{\mathit{admasp}_{#1}}}

\newcommand{\varass}{\ensuremath{T}}

\newcommand{\lastpref}{\ensuremath{E}}

\newcommand{\restrict}[2]{\ensuremath{{#1}_{\downarrow {#2}}}}

\newcommand{\compo}{\ensuremath{\circ}}
\newcommand{\Setmin}[2]{\ensuremath{S^{min}_{#1}(#2)}}
\newcommand{\Setmax}[2]{\ensuremath{S^{max}_{#1}(#2)}}
\newcommand{\Cardmin}[2]{\ensuremath{C^{min}_{#1}(#2)}}
\newcommand{\Cardmax}[2]{\ensuremath{C^{max}_{#1}(#2)}}
\newcommand{\nosupset}[1]{\ensuremath{\mathcal{N}^{\supseteq}(#1)}}
\newcommand{\nosubset}[1]{\ensuremath{\mathcal{N}^{\subseteq}(#1)}}

\newcommand{\Cardminasp}[2]{\ensuremath{C^{min}_{#1,asp}(#2)}}
\newcommand{\Cardmaxasp}[2]{\ensuremath{C^{max}_{#1,asp}(#2)}}
\newcommand{\nosupsetasp}[1]{\ensuremath{\mathcal{N}^{\supseteq}_{asp}(#1)}}
\newcommand{\nosubsetasp}[1]{\ensuremath{\mathcal{N}^{\subseteq}_{asp}(#1)}}

\newcommand{\NP}{\ensuremath{\rm NP}}
\newcommand{\SigmaP}[1]{\ensuremath{{\Sigma}_{#1}^{P}}}

\newcommand{\CONP}{\ensuremath{\mbox{\rm co-}\NP}}

\newcommand{\DeltaP}[1]{\ensuremath{{\Delta}_{#1}^{P}}}

\newcommand{\nop}[1]{}

\begin{document}

\title{Solving Set Optimization Problems by Cardinality Optimization via Weak Constraints with an Application to Argumentation\thanks{An extended version of this paper will appear at ECAI 2016 \cite{fabe-etal-2016-ecai}.}} 
\author{Wolfgang Faber\inst{1} \and Mauro Vallati\inst{1} \and Federico Cerutti\inst{2} \and Massimiliano Giacomin\inst{3}}

\institute{
  University of Huddersfield, UK\\
  \email{n.surname@hud.ac.uk}\\
  \and
  Cardiff University, UK\\
  \email{ceruttif@cardiff.ac.uk}\\
  \and
  Universit\`a degli Studi di Brescia, Italy\\
  \email{massimiliano.giacominb@unibs.it}
}

\maketitle
\bibliographystyle{splncs03}

\begin{abstract}

Optimization---minimization or maximization---in the lattice of subsets is a frequent operation in Artificial Intelligence tasks. Examples are subset-minimal model-based diagnosis, nonmonotonic reasoning by means of circumscription, or preferred extensions in abstract argumentation. Finding the optimum among many admissible solutions is often harder than finding admissible solutions with respect to both computational complexity and methodology. This paper addresses the former issue by means of an effective method for finding subset-optimal solutions. It is based on the relationship between cardinality-optimal and subset-optimal solutions, and the fact that many logic-based declarative programming systems provide constructs for finding cardinality-optimal solutions, for example maximum satisfiability (MaxSAT) or weak constraints in Answer Set Programming (ASP). Clearly each cardinality-optimal solution is also a subset-optimal one, and if the language also allows for the addition of particular restricting constructs (both MaxSAT and ASP do) then all subset-optimal solutions can be found by an iterative computation of cardinality-optimal solutions. As a showcase, the computation of preferred extensions of abstract argumentation frameworks using the proposed method is studied.


\end{abstract}

\section{Introduction}
In Artificial Intelligence, the task of set optimization, in the sense of finding a set that is minimal or maximal with respect to set inclusion, frequently occurs. There are famous examples such as Circumscription \cite{McCarthy1980} or Model-based Diagnosis that involve set minimization. Computing preferred extensions of abstract argumentation frameworks \cite{dung1995} is an example that involves set maximization.

Often, set optimization is an element that creates difficulties in implementation and representation. For example, McCarthy had to resort to Second-order logic for defining Circumscription of First-order theories \cite{McCarthy1980}. Also for computing preferred extensions, relatively sophisticated techniques are required, for instance QBFs rather than propositional formulas \cite{Arieli2013}.

There is another notion of set optimization, finding a set that has minimal or maximal cardinality, which has more readily available system support nowadays. The most prominent examples of languages and systems that support cardinality optimization are MaxSAT, Constraint Programming, and Answer Set Programming (ASP).

In this paper, we show how set optima can be computed by a general algorithm that employs cardinality optimizing subroutines, provided that the underlying languages allow for expressing simple constraints. 
Algorithm MCSes in Figure~2 of \cite{liffiton2008algorithms}, which computes Minimal Correction Sets\footnote{It is worth noticing that algorithms exploiting minimal correction sets have been proposed for computing argumentation semantics extensions, in particular for semi-stable, ideal, and eager semantics \cite{Wallner2013}, but not for preferred semantics which is our main test-case in this paper.} of a propositional formula, bears a few similarities to our algorithms. However, it is also different in several respects, most notably it is formulated for solving one particular problem only, assumes propositional formulas as the representation formalism, and does not employ a cardinality optimization oracle explicitly. To the best of our knowledge a general method applicable in a variety of representation formalisms and for unspecific set optimization problems settings has not been proposed in the literature before. We develop an instantiation of the general method for ASP, and show that it is suitable. In~\cite{fabe-etal-2016-ecai}, we also present an instantiation using MaxSAT.

There are two more recent software tools that also support computing set optimization problems when the underlying language is ASP: {\bf asprin} \cite{asprin} and {\bf D-FLAT\textasciicircum 2} \cite{dflat}. The scope of the tool {\bf asprin} is actually reasoning with preferences, but as a special case one can express preferences such that only set optimal solutions remain. The required preferences come with the predefined library of {\bf asprin}. The underlying algorithms of {\bf asprin} are very different from those presented in this paper. {\bf D-FLAT\textasciicircum 2} builds on dynamic programming and exploits tree decomposition in order to solve set optimization problems and is therefore also very different from the method that we will present in this paper. In the realm of ASP, the tool {\bf metasp} \cite{Gebseretal2011} can be seen as a predecessor of {\bf asprin}, which does not seem to be maintained any longer. It relies on reification of rules and exploits a programming pattern known as saturation for set optimization, which is also very different from the method described in this paper.

We then turn our attention to a showcase application, computing preferred extensions of abstract argumentation frameworks. Dung's theory of abstract argumentation \cite{dung1995} is a unifying framework
able to encompass a large variety of specific formalisms in the areas of
nonmonotonic reasoning, logic programming and computational argumentation.
It is based on the notion of argumentation framework (\AFname),
consisting of a set of \emph{arguments} and a binary \emph{attack} relation between them.
Arguments can thus be represented by nodes of a directed graph, and attacks by arcs. The nature of arguments is left unspecified: it can be anything from logical statements to informal natural language text. For instance, \cite{Toniolo2015} shows how argumentation can be efficiently used for supporting critical thinking and intelligence analysis in military-sensitive contexts.

Different \emph{argumentation semantics} declare the criteria to determine
which arguments emerge as ``justified'' among conflicting ones,
by identifying a number of \emph{extensions}, \ie{} sets of arguments that can \virg{survive the conflict together}.
In \cite{dung1995} four \virg{traditional} semantics were introduced, namely \emph{complete}, \emph{grounded}, \emph{stable}, and \emph{preferred} semantics. 
For a complete overview of subsequently proposed alternative semantics, the interested reader is referred to \cite{KER2011}.

The main computational problems in abstract argumentation
include \emph{decision}---\eg{} determine if an argument is in all the extensions prescribed by a semantics---and \emph{construction} problems,
and turn out to be computationally intractable for most argumentation semantics \cite{dw:2009}.
In this paper we focus on the \emph{extension enumeration} problem,
\ie{} constructing \emph{all} extensions for a given \AFname: its solution  provides complete information about the justification status of arguments
and allows for solving the other problems as well.

Our general method allowed for the definition and implementation of two novel algorithms for enumerating preferred extensions: \asp{} (based on an ASP solver) and \maxsat{} (based on a MaxSAT solver). The former is described in this paper and available on \url{http://www.wfaber.com/software/prefASP/}, the latter in~\cite{fabe-etal-2016-ecai}. Both are evaluated using benchmarks from the International Competition on Computational Models of Argumentation (ICCMA2015). We report on a variety of experiments: the first focuses on \asp\ and starts with comparing the use of different solver configurations for \asp{}, followed by a comparison of \asp{} to {\bf asprin} and {\bf D-FLAT\textasciicircum 2}, and eventually comparing \asp{} to the dedicated argumentation solver {\bf ASPARTIX-V}. Eventually we compare \asp{} and \maxsat{} to the ICCMA2015 competition winner \winner{}. The experiments show that despite their conceptual simplicity, our software tools are competitive with the best available ones.




\section{Abstract Methodology} \label{sec:meth-abstract}

The proposed methodology applies to a variety of knowledge representation formalisms, we therefore consider an abstract setting. We define a knowledge base $\K$ to be associated with a set $\sol{\K}$ of solutions, and we assume that each $s \in \sol{\K}$ is a set. We also use a set restriction operator $\restrict{}{O}$ such that $\restrict{s}{O} = s \cap O$, the idea being that $\restrict{s}{O}$ identifies the solution elements that are relevant for optimization. We also assume a composition operator $\K_1 \compo \K_2$ to be present that allows to compose two knowledge bases $\K_1$ and $\K_2$ (intended as merging two knowledge bases; for bases represented as sets, $\compo$ will usually be $\cup$). We next define a few optimization criteria for solutions of knowledge bases.

\begin{definition}
Let $\K$ be a knowledge base and $R$ be a set (of elements occuring in solutions of $\K$). Define:
\begin{align*}
\Setmax{R}{\K} & = \{s \mid s \in \sol{\K}, \nexists s' \in \sol{\K}: \restrict{s'}{R} \supset \restrict{s}{R}\}\\
\Setmin{R}{\K} & = \{s \mid s \in \sol{\K}, \nexists s' \in \sol{\K}: \restrict{s'}{R} \subset \restrict{s}{R}\}\\
\Cardmax{R}{\K} & = \{s \mid s \in \sol{\K}, \nexists s' \in \sol{\K}: |\restrict{s'}{R}| > |\restrict{s}{R}|\}\\
\Cardmin{R}{\K} & = \{s \mid s \in \sol{\K}, \nexists s' \in \sol{\K}: |\restrict{s'}{R}| < |\restrict{s}{R}|\}
\end{align*}
\end{definition}

While $\Setmin{R}{\K}$ and $\Setmax{R}{\K}$ occur in diverse applications of knowledge representation and reasoning, it often happens that the computational complexity of these tasks increases (under standard assumptions) compared to $\Cardmin{R}{\K}$ and $\Cardmax{R}{\K}$. For example, deciding whether $\sol{\K} = \emptyset$ is \CONP{}-complete (in the size of $\K$) if $\K$ is represented using a propositional formula and $\sol{\K}$ is the set of satisfying assignments, where each assignment is represented as the set of true variables. In this setting, computing $\Setmax{R}{\K}$ is then \SigmaP{2}-hard, as showing the optimality of a solution may require exponentially many \CONP{} checks, while $\Cardmax{R}{\K}$ is in \DeltaP{2}, requiring at most a polynomial number of \CONP{} checks.
Note that this does not necessarily have practical consequences, because at the moment all known algorithms to solve these problems require at least exponential time.

There are also representational repercussions. Still assuming $\K$ to be a propositional formula, one cannot find a propositional formula of polynomial size that encodes any of $\Setmin{R}{\K}$, 
$\Setmax{R}{\K}$, $\Cardmax{R}{\K}$, and $\Cardmin{R}{\K}$ (if $NP\neq\SigmaP{2}$, which is currently unknown, but often conjectured).
If $\K$ has been modelled using ASP, it is easy to encode $\Cardmax{R}{\K}$ and $\Cardmin{R}{\K}$ because of the availability of weak constraints (or optimization constructs). In fact, one can use ASP also for encoding $\Setmin{R}{\K}$ and
$\Setmax{R}{\K}$, because ASP can express all problems in $\SigmaP{2}$. We will discuss this further in Section~\ref{sec:meth-maxsat-asp}.


In this paper, we relate $\Setmax{R}{\K}$ to $\Cardmax{R}{\K}$ (and $\Setmin{R}{\K}$ to $\Cardmin{R}{\K}$). We first observe that each cardinality optimal solution is also subset optimal.
\begin{observation}\label{obs:cardoptissetopt}
For any knowledge base $\K$ and set $R$, $\Cardmax{R}{\K} \subseteq \Setmax{R}{\K}$ and $\Cardmin{R}{\K} \subseteq \Setmin{R}{\K}$.
\end{observation}
This observation holds because if any $s \in \Cardmax{R}{\K}$ were not in $\Setmax{R}{\K}$, then there would be some $s' \in \sol{\K}$ such that $\restrict{s'}{R} \supset \restrict{s}{R}$ and clearly $|\restrict{s'}{R}| > |\restrict{s}{R}|$ then holds (and symmetrically for minimization).

This implies that when the task is to compute one subset optimal solution, one can instead safely compute one cardinality optimal solution. When, however, the computational task involves an enumeration of all subset optimal solutions, one is faced with incompleteness, as not all subset optima are cardinality optimal.

\begin{example}\label{ex:setmaxcardmax}
Let $\K_1$ be such that $\sol{\K_1} = \{\{a,b\}, \{b\}, \{c\}\}$ and let $R_1 = \{a,b,c\}$. Then $\Setmax{R_1}{\K_1} = \{\{a,b\},\{c\}\}$ while $\Cardmax{R_1}{\K_1} = \{\{a,b\}\}$.
\end{example}

This can be overcome by an iterative approach, in which first cardinality optimal solutions are computed. In the next stage, the knowledge base is extended in a way that it no longer admits the solutions already found or any subsets (for maximization) or supersets (for minimization) thereof. 

\begin{definition}
Given a knowledge base $\K$, a set $R$, and a set $S \subseteq \sol{\K}$, let $\nosubset{\K,R,S}$ denote a knowledge base such that
$$\sol{\K \compo \nosubset{\K,R,S}} = \sol{\K} \setminus \{s'\mid \restrict{s'}{R} \subseteq \restrict{s}{R} \land s \in S\}.$$
Symmetrically, let $\nosupset{\K,R,S}$ denote a knowledge base such that
$$\sol{\K \compo \nosupset{\K,R,S}} = \sol{\K} \setminus \{s'\mid \restrict{s'}{R} \supseteq \restrict{s}{R} \land s \in S\}.$$
\end{definition}

It depends on the formalism used for the knowledge base, whether $\nosubset{\K,R,S}$ and $\nosupset{\K,R,S}$ can be created, and in particular whether they can be represented in a concise way. It also depends on the formalism whether there is a uniform way of encoding $\nosubset{\K,R,S}$ and $\nosupset{\K,R,S}$, or whether one has to rely on a representation that depends on the structure of $\K$.

The iterative approach is formalized for subset maximality in Algorithm~\ref{alg:setmaxbycardmax}, it is easily adapted to subset minimality.
Note that practical algorithms will usually not collect all solutions in the output because of space considerations, but rather output them immediately as they are computed.

\begin{algorithm}[t]
  \caption{Enumerating $\Setmax{R}{\K}$ by means of $\Cardmax{R}{\K}$}\label{alg:setmaxbycardmax}
   \begin{algorithmic}[1]
    \STATE \textbf{Input:} $\K, R$
    \STATE \textbf{Output:} $\Setmax{R}{\K}$
    \STATE $\K_i~:=~\K$
    \STATE $S~:=~\emptyset$
    \STATE $S_i~:=~\Cardmax{R}{\K_i}$
    \WHILE{$S_i~!=~\emptyset$}
    \STATE $S~:=~S \cup S_i$
    \STATE $\K_i~:=~\K_i \compo \nosubset{\K_i,R,S_i}$
    \STATE $S_i~:=~\Cardmax{R}{\K_i}$
    \ENDWHILE
    \STATE \textbf{return} $S$ 
   \end{algorithmic}
\end{algorithm}

\nop{
\begin{algorithm}[t]
  \caption{Enumerating $\Setmin{R}{\K}$ by means of $\Cardmin{R}{\K}$}\label{alg:setminbycardmin}
   \begin{algorithmic}[1]
    \STATE \textbf{Input:} $\K, R$
    \STATE \textbf{Output:} $\Setmin{R}{\K}$
    \STATE $\K_i~:=~\K$
    \STATE $S~:=~\emptyset$
    \STATE $S_i~:=~\Cardmin{R}{\K_i}$
    \WHILE{$S_i~!=~\emptyset$}
    \STATE $S~:=~S \cup S_i$
    \STATE $\K_i~:=~\K_i \compo \nosupset{\K_i,R,S_i}$
    \STATE $S_i~:=~\Cardmin{R}{\K_i}$
    \ENDWHILE
    \STATE \textbf{return} $S$ 
   \end{algorithmic}
\end{algorithm}
}

\begin{theorem}
For a knowledge base $\K$ and set $R$, Algorithm~\ref{alg:setmaxbycardmax} computes $\Setmax{R}{\K}$.
\end{theorem}
\begin{proof}[Sketch]
We first observe that each assignment of variable $S_i$ contains only elements of $\Setmax{R}{\K}$. When variable $S_i$ is first initialized in line~5 of Algorithm~\ref{alg:setmaxbycardmax}, Observation~\ref{obs:cardoptissetopt} guarantees the claim. For each later assignment, by construction only $s \in \sol{\K}$ are assigned, and any such $s$ is such that $\nexists s' \in \sol{\K}: \restrict{s'}{R} \supset \restrict{s}{R}$  (otherwise $|\restrict{s'}{R}| > |\restrict{s}{R}|$ would hold). It is also clear that the algorithm terminates (if $\sol{\K}$ is finite).

Now observe that each $s \in \Setmax{R}{\K}$ is assigned once to $S_i$ in Algorithm~\ref{alg:setmaxbycardmax}. Indeed, the first assignment contains all elements in $\Setmax{R}{\K}$ that are of maximum cardinality, the next iteration contains all elements in $\Setmax{R}{\K}$ of the next-highest cardinality, and so forth down to the elements of $\Setmax{R}{\K}$ of least cardinality in the last assignment. In this way, all elements of $\Setmax{R}{\K}$ will be contained in $S$ when Algorithm~\ref{alg:setmaxbycardmax} terminates.
\end{proof}

It should be pointed out that Algorithms~\ref{alg:setmaxbycardmax} 
also work when instead of  $\Cardmax{R}{\K_i}$ (or $\Cardmin{R}{\K_i}$) any non-empty subset thereof is assigned to $S_i$ in lines 5 and 9.

Let us note that the number of subcalls to $\Cardmax{R}{\K_i}$ (resp., $\Cardmin{R}{\K_i}$) is at most $|\Setmax{R}{\K}|$ (resp., $|\Setmin{R}{\K}|$). The cardinality of these sets can be exponential in $\K$ in the worst case. That also means that in the worst case an exponential number of knowledge bases \nosupset{\K_i,R,S_i} (or \nosubset{\K_i,R,S_i}) are composed to $\K$, which could lead to a use of exponential space. Note however, that this only occurs if there is an exponential number of solutions to be generated by the algorithm.
This only occurs if $\restrict{s_1}{R} \nsubseteq \restrict{s_2}{R}$ and $\restrict{s_2}{R} \nsubseteq \restrict{s_1}{R}$ holds for almost all solutions $s_1$ and $s_2$ of $\sol{\K}$.

Note that there is also a contrast to more traditional algorithms that invoke a $\CONP$ oracle call for each $s \in \sol{\K}$, especially if they run a test on each subset of a found solution. In that setting, the number of subcalls that take exponential time will usually be much greater than $|\Setmin{R}{\K}|$ (resp., $|\Setmin{R}{\K}|$). We view this feature of our algorithm as one of the main advantages over more traditional methods.

Algorithm~\ref{alg:setmaxbycardmax} bears some similarities to Algorithm MCSes in Figure~2 of \cite{liffiton2008algorithms}. Algorithm MCSes computes Minimal Correction Sets of a propositional formula, and it solves a much more specific problem and assumes a specific knowledge representation formalism. 
In fact, it iteratively increases the cardinality of the (relevant portion of the) solution to be computed and enforces the cardinality by means of formulas thus imitating the behaviour of a MaxSAT algorithm.
The clauses that Algorithm \mbox{MCSes} adds follow the same idea of $\nosubset{\K_i,R,S_i}$.

\section{Concretizations Using ASP} \label{sec:meth-maxsat-asp}

We now show how to instantiate the abstract method described in Section~\ref{sec:meth-abstract} using ASP. In Answer Set Programming (ASP) one asks for the answer sets (often also called stable models) of a logic program. The full language specification can be found at \url{https://www.mat.unical.it/aspcomp2013/ASPStandardization}, below we provide a brief overview of the concepts relevant to this work.

The basic constructs in ASP logic programs are of the form
$$\mathtt{h_1\ |\ \ldots\ |\ h_k\ :-\ b_1,\ \ldots,\ b_m,\ not\ b_{m+1},\ \ldots,\ not\ b_n.} $$
where $\mathtt{0 \leq k}$, $\mathtt{0 \leq m \leq n}$ and the $\mathtt{h_i}$ and $\mathtt{b_j}$ are function-free first-order atoms. When $\mathtt{k} > 0$, it is called a rule, otherwise a constraint. If $\mathtt{k} = 1$ and $\mathtt{m} = \mathtt{n} = 0$, the rule is called a fact. The part left of the construct $\mathtt{:-}$ is called head, the part right of it is called body. Sets of rules and constraints are called programs.

Programs with variables are thought of as shorthand for their ground (variable-free) versions with respect to the Herbrand universe of the program. Answer sets are defined on ground programs: they are Herbrand models of the program, which satisfy an additional stability condition. In the following we will assume $\proplang$ to be the set of all ground atoms.

The language of ASP consists of quite a lot more constructs. One relevant for this paper is the weak constraint, which takes the form
$$\mathtt{:\sim\ b_1,\ \ldots,\ b_m,\ not\ b_{m+1},\ \ldots,\ not\ b_n.} $$
An interpretation that satisfies all literals to the right of $\mathtt{:\sim}$ will incur a (uniform) penalty. Answer sets of programs with weak constraints are then those answer sets of the weak-constraint-free portion that minimize the penalties incurred by weak constraints.

For ASP, the terminology of Section~\ref{sec:meth-abstract} is instantiated as follows: $\K$ in this setting is a weak-constraint-free program, and $\sol{\K}$ is the set of its answer sets. The operation $\K_1 \compo \K_2$ simply is the set union of the two programs $\K_1$ and $\K_2$.

It is possible to encode $\Cardmax{R}{\K}$ and $\Cardmin{R}{\K}$ by means of weak constraints.

\begin{definition}
Given a set of clauses $\K$ and $R \subset \proplang$,  we define 
\[
\Cardmaxasp{R}{\K} = \K \cup \{ \mathtt{:\sim\ not}\ r. \mid r \in R\}
\]
and in a similar way
\[
\Cardminasp{R}{\K} = \K \cup \{\mathtt{:\sim}\ r. \mid r \in R\}~~.
\]
\end{definition}

It is easy to verify that $\Cardmax{R}{\K}$ corresponds to the answer sets of \Cardmaxasp{R}{\K} and $\Cardmin{R}{\K}$ corresponds to the answer sets of \Cardminasp{R}{\K}.

ASP also allows for encoding $\Setmax{R}{\K}$ and $\Setmin{R}{\K}$, but this requires rather involved, and often ad-hoc programs. A general approach has been presented in \cite{Gebseretal2011}, but it relies on reification techniques, which is often detrimental for performance.

Let us now consider how to obtain $\nosubset{\K,R,S}$ and $\nosupset{\K,R,S}$ in ASP. For $\nosubset{\K,R,S}$, one requires for each solution in $S$ that not all elements of $R$ outside that solution should be false. This inhibits the solution itself and any subset (restricted to $R$) of it.
For $\nosupset{\K,R,S}$, we require for each solution in $S$ that not all of its elements in $R$ should be true. In this way the solution itself and any superset is inhibited.

\begin{definition}
Given a set of clauses $\K$, $R \subset \proplang$,  and $S \subseteq \sol{\K}$, let
\[
\nosubsetasp{\K,R,S} = \{\mathtt{:-\ not\ a_1,\ \ldots,\ not\ a_n.} \mid s \in S, R \setminus s = \{\mathtt{a_1},\ldots,\mathtt{a_n}\} \}
\]
and in a similar way
\[
\nosupsetasp{\K,R,S} = \{\mathtt{:-\  a_1,\ \ldots,\ a_n.} \mid s \in S, R \cap s = \{\mathtt{a_1},\ldots,\mathtt{a_n}\} \}~~.
\]
\end{definition}

Again, it is easy to verify that \nosubsetasp{\K,R,S} and \nosupsetasp{\K,R,S} restrict the answer sets in the way required by \nosubset{\K,R,S} and \nosupset{\K,R,S}.

\section{Methodology Showcase: Abstract Argumentation}\label{approach}

In this section we show how to use the proposed methodology for enumerating all preferred extensions of abstract argumentation frameworks. After a short background on abstract argumentation, we introduce an ASP-based solver (\asp{}) that employs the methods described in Section~\ref{sec:meth-maxsat-asp}.

\subsection{Background on Abstract Argumentation}

An argumentation framework \cite{dung1995} consists of a set of arguments\footnote{In this paper we consider only \emph{finite} sets of arguments: see \cite{Baroni2013} for a discussion on infinite sets of arguments.} and a binary attack relation between them.

\begin{definition}
An \emph{argumentation framework} (\AFname) is a pair $\anAFsymbol = \anAF$
where $\setargs$ is a set of arguments and $\setattacks \subseteq \setargs \times \setargs$.
We say that \argb{} \emph{attacks} \arga{} iff $\tuple{\argb,\arga} \in \setattacks$, also denoted as $\attacks{\argb}{\arga}$.
The set of attackers of an argument $\arga$ will be denoted as
$\attackers{\arga} \triangleq \set{\argb : \attacks{\argb}{\arga}}$,
the set of arguments attacked by $\arga$ will be denoted as
$\attacked{\arga} \triangleq \set{\argb : \attacks{\arga}{\argb}}$.
\end{definition}


Each \AFname{} has an associated directed graph where the vertices are the arguments, and the edges are the attacks.

As an intuitive example from \cite{Besnard2014a}, let $\arga$ be the argument ``Patient has hypertension so prescribe diuretics;'' $\argb$: ``Patient has hypertension so prescribe betablockers;'' and $\argc$: ``Patient has emphysema which is a contraindication for betablockers.''
Intuitively, assuming that only one treatment is possible at the very same time, $\arga$ attacks $\argb$ and vice versa, while $\argc$ suggests that $\argb$ should not be the case ($\argc$ attacks $\argb$).
Therefore, let $\AFsymbol{M} = \AF{M}$ such that, $\setargs_M = \set{\arga, \argb, \argc}$ and $\setattacks_M = \set{\tuple{\argc, \argb}, \tuple{\argb, \arga}, \tuple{\arga, \argb}}$. $\AFsymbol{M}$ is depicted in Fig. \ref{fig:ex}.

\begin{figure}[t]
  \centering
      \begin{tikzpicture}[->,>=stealth',shorten >=1pt,auto,node
      distance=1.2cm, thick,main node/.style={draw=none,fill=none}]
      
      \node[main node] (1) {$\argc$};
      \node[main node] (2) [right of=1] {$\argb$};
      \node[main node] (3) [right of=2] {$\arga$};

      \tikzset{edge/.style = {->,> = latex'}} 
      
      \draw[edge] (1) to (2);
      \draw[edge] (2) to[bend right] (3);
      \draw[edge] (3) to[bend right] (2);
    \end{tikzpicture}

  \caption{The \AFname{} $\AFsymbol{M}$ for the hypertension problem}
  \label{fig:ex}
\end{figure}

The basic properties of \dungconffree ness, acceptability, and admissibility of a set of arguments
are fundamental for the definition of argumentation semantics.

\begin{definition}
\label{def:recall}
Given an $\AFname$ $\anAFsymbol = \anAF$:
\begin{itemize}
  \item a set $\aset \subseteq \setargs$ is a \emph{\dungconffree} set of $\anAFsymbol$
        if $\nexists~ \arga, \argb \in \aset$ s.t. $\attacks{\arga}{\argb}$;
  \item an argument $\arga \in \setargs$ is \emph{\dungacceptable} with respect to a set $\aset \subseteq \setargs$ of  $\anAFsymbol$
        if $\forall \argb \in \setargs$ s.t. $\attacks{\argb}{\arga}$,
        $\exists~ \argc \in \aset$ s.t. $\attacks{\argc}{\argb}$;
  \item the function $\carfun{\anAFsymbol}: 2^{\setargs} \rightarrow 2^{\setargs}$ such that
        $\carfun{\anAFsymbol}(\aset) = \set{\arga \mid \arga \mbox{ is } \mbox{\dungacceptable} \mbox{ w.r.t. } \aset}$
        is called the \emph{\dungcharacteristic{} function} of \anAFsymbol;

  \item a set $\aset \subseteq \setargs$ is an \emph{\dungadmissible{} set} of $\anAFsymbol$
        if $\aset$ is a \dungconffree{} set of $\anAFsymbol$ and every element of $\aset$
        is \dungacceptable{} with respect to $\aset$ of $\anAFsymbol$.
\end{itemize}
\end{definition}

In the AF $\AFsymbol{M}$ of Fig. \ref{fig:ex}, $\set{\arga}$ is an admissible set because it is \dungconffree{} (there is no such attack $\tuple{\arga, \arga}$) and each element of the set (i.e. $\arga$) is defended against the attack it receives (i.e. $\arga$ is attacked by $\argb$, but, in turn, $\argb$ is attacked by $\arga$).

An argumentation semantics $\gensem$ prescribes for any \AFname{} $\anAFsymbol$ a set of \emph{extensions}, denoted as $\setgenext{\gensem}{\anAFsymbol}$, namely a set of sets of arguments satisfying the conditions dictated by $\gensem$.
Here we recall the definition of the grounded semantics, denoted as $\GR$, and of the preferred semantics, denoted as $\PR$.

\begin{definition}
\label{def_sem_recall}
Given an $\AFname$ $\anAFsymbol = \anAF$:

\begin{itemize}
\item a set $\aset \subseteq \setargs$ is the \emph{\dunggrounded{} extension}  of $\anAFsymbol$, 
        if $\aset$ is the least (w.r.t. set inclusion) fixed point of the \dungcharacteristic{} function $\carfun{\anAFsymbol}$;        

\item a set $\aset \subseteq \setargs$ is 
 a \emph{\dungpreferred{} extension} of $\anAFsymbol$,
             \ie{} $\aset \in \setgenext{\PR}{\anAFsymbol}$,
              if $\aset$ is a maximal (w.r.t. \ensuremath{\subseteq}) \dungadmissible{} set of $\anAFsymbol$.

\end{itemize}
\end{definition}

While $\set{\arga}$ is an \dungadmissible{} set for $\AFsymbol{M}$, it is not a \dungpreferred{} extension. In fact, $\set{\argc, \arga}$ is also an admissible set which contains $\set{\arga}$. Since there are no admissible supersets of $\set{\argc, \arga}$, it is therefore maximal and thus a preferred extension, the only one for $\AFsymbol{M}$.

\subsection{Admissible Extensions in ASP}

For ASP, an encoding for admissible extensions is rather straightforward, see \cite{Egly2010,Charwat2015}.

\begin{definition}
\label{def:admasp}
  Given an $\AFname$ $\anAFsymbol = \anAF$,  for each $a \in \setargs$ a fact
  $\mathtt{arg(a).}$
  is created and for each $(a,b) \in \setattacks$ a fact
  $\mathtt{att(a,b).}$
  is created (this corresponds to the apx file format in the ICCMA competition). 
  Together with the program
  \begin{align*}
&\mathtt{in(X) :- not\ out(X), arg(X).}\\
&\mathtt{out(X) :- not\ in(X), arg(X).}\\
&\mathtt{:- in(X), in(Y), att(X,Y).}\\
&\mathtt{defeated(X) :- in(Y), att(Y,X).}\\
&\mathtt{not\_defended(X) :- att(Y,X), not\ defeated(Y).}\\
&\mathtt{:- in(X), not\_defended(X).}
\end{align*}
we form $\admaspenc{\anAFsymbol}$ and there is a one-to-one correspondence between answer sets of $\admaspenc{\anAFsymbol}$ and admissible extensions.
\end{definition}



\subsection{Preferred Extensions via Algorithm~\ref{alg:setmaxbycardmax}}

We can now use our methodology in order to step from admissible to preferred extensions. Indeed, if $\K$ encodes admissible extensions of an AF and $R$ is the language part encoding the extensions, then $\Setmax{R}{\K}$ encodes preferred extensions. We can then use Algorithm~\ref{alg:setmaxbycardmax} to compute them. This gives rise to two solvers, $\asp$ and $\maxsat$ (the latter being described in \cite{fabe-etal-2016-ecai}).

For $\asp$, given an $\AFname$ $\anAFsymbol = \anAF$, we use Algorithm~\ref{alg:setmaxbycardmax} with input $\K$ being $\admaspenc{\anAFsymbol}$ and input $R$ being $\{\mathtt{in(a)}\mid \mathtt{a} \in \setargs\}$, in the following referred to as $I(\setargs)$. In lines 5 and 9 of Algorithm~\ref{alg:setmaxbycardmax}, we use an ASP solver for obtaining all answer sets of $\Cardmaxasp{I(\setargs)}{\K_i}$. In line 8, $\nosubsetasp{\K_i,I(\setargs),S_i}$ is used.

Apart from the encodings and underlying solvers, there is also a difference in the fact that $\asp$ computes all cardinality-optimal solutions in one go, while $\maxsat$ computes one at a time.

\section{Experimental Analysis}\label{exp}

In order to evaluate the efficiency of the
introduced algorithms, we have carried out an experimental analysis where performance is analyzed from different perspectives. After describing the general setup, we first report on an experiment for choosing a parameter setting in the backend solver of \asp{}.
Next we report on a comparison of \asp{} with {\bf asprin} and {\bf D-FLAT\textasciicircum 2}. As discussed in the Introduction, these systems also support set optimization in an ASP setting, but use very different underlying algorithms. This is followed by a comparison to dedicated argumentation systems, in which we compare \asp{} with {\bf ASPARTIX-V}, since {\bf ASPARTIX-V} is also based on ASP. Finally, we compare \asp{} and \maxsat{} with \winner{}, which is the state-of-the-art solver, in the sense that it won the ICCMA2015 competition.
These comparisons are deliberately split up into small pairings, in order to have a crisper picture of the relative performance measure. Indeed, the IPC score depends on the solvers considered in the comparison, while PAR10 and coverage are specific to a single solver. Experiments do not need to be re-run for presenting results together: however, comparisons would be less informative given the changes in the IPC score.



\subsection{Experimental Settings}\label{sec:expset}

The ASP-based algorithm \asp{} has been implemented as a bash script using basic tools like sed and grep and exploits clingo 4.5.2 \cite{clingo} as ASP solver. \maxsat{} has been implemented in C++, and exploits the ASPino MaxSAT solver \cite{aspino}.

The experiments were conducted on a cluster with computing nodes equipped with 2.5 GHz Intel{\texttrademark} Core 2 Quad Processors, 4 GB of RAM and Linux operating system. 
A cutoff of 900 seconds---wallclock time---was imposed to compute the preferred extensions for each \AFname. 
For each involved solver we recorded the overall result: success (if it finds each preferred extension), crashed, timed-out or ran out of memory. In fact, in our experimental evaluation all the unsuccessful runs are due to time-out. Experiments  have  been  conducted  on  the  ICCMA2015  benchmarks \cite{thimm2016summary}, which is a set of randomly generated 192 \AFname{}s. They have been generated considering three different graph models, in order to provide different levels of complexity. More details can be found on the ICCMA website.\footnote{http://argumentationcompetition.org}


The performance measures reported in this paper are the Penalised Average Runtime and the International Planning Competition (IPC) score. 

The
Penalised Average Runtime (PAR score) is a real number which counts (i) runs that  fail to find all the preferred extensions as ten times the cutoff time (PAR10) and (ii) runs that succeed as the actual runtime.
PAR scores are commonly used in automated algorithm configuration, algorithm selection, and portfolio construction, because using them
allows runtime to be considered while still placing a strong emphasis on high instance set coverage. 

The IPC score, borrowed from the Planning community, is defined as follows:

\begin{itemize}
\item For each test case (in our case, each test \AFname) let $T^*$ be the best execution time among the compared systems (if no system produces the solution within the time limit, the test case is not considered valid and ignored).

\item For each valid case, each system gets a score of $1/(1 + \log_{10}(T/T^*))$, where $T$ is its execution time, or a score of 0 if it fails in that case. Runtimes below 1 second get by default the maximal score of 1.
\end{itemize}

The IPC score for a system is the sum of its scores over all the valid test cases. 

It should be noted that the IPC score depends on the ensemble of tested systems, that is, it is a relative measure, which  depends on the experimental context. Indeed, in the following, the IPC scores of \asp{} vary depending on the different experimental settings. In contrast to this, the PAR10 score of \asp{} remains equal, as this is an absolute score.


\subsection{Comparison of \asp{} Using Different Solver Configurations}

Clingo offers several different solver configurations (inherited from the clasp solver used in clingo) which correspond to different heuristic and search setups, see \cite{Gebser2015clasp} for detailed explanations. As a first analysis, we investigated how robust \asp{} is with respect to these configurations, and, as a by-product, we determined the best-performing configuration to be used for comparison with to other systems.

\begin{table}
\centering
\caption{\label{tab:asp} Comparison of different Clingo solver configurations, that can be exploited by \asp{}, on the ICCMA2015 benchmarks. Results are shown in terms of IPC score (maximum achievable is 192.0), percentages of success (\% Success) and PAR10.} 
\renewcommand{\arraystretch}{1.2}
\begin{tabular}{ l r r r }
\toprule

& IPC score & \% Success & PAR10 \\
 \midrule
\textbf{Crafty} & 183.7 & 100.0 & 23.1 \\
\hdashline
\textbf{Frumpy} & 178.6 & 99.5 & 75.0 \\
\hdashline
\textbf{Jumpy} & 172.0 & 99.5 & 82.3 \\
\hdashline
\textbf{Many} & 177.4 & 99.5 & 84.8 \\
\bottomrule

\end{tabular}

\end{table}

Table \ref{tab:asp} shows a comparison of the different solver configurations, in terms of IPC score, percentage of successfully analysed \AFname{}S and PAR10, on the ICCMA2015 benchmarks. It can be observed that all configurations perform relatively similar to each other, implying that the particular chosen configuration is not critical for the performance of \asp{}.

However, there is one winning configuration: the {\bf Crafty} configuration allows \asp{} to enumerate preferred extensions of all the considered \AFname{}s, and to provide solutions faster. This configuration is geared towards ``crafted'' problems, which also makes sense in the context of the considered benchmarks. Therefore, in the rest of the experimental analysis the {\bf Crafty} configuration will be considered for \asp{}.

\subsection{Comparison with Existing General Algorithms}

We now turn to general tools that allow for easy representation and effective solution of subset optimization problems. To the best of our knowledge, the only tools of this kind are {\bf asprin} \cite{asprin} (with its predecessor {\bf metasp}) and {\bf D-FLAT\textasciicircum 2} \cite{dflat}, which we have discussed in the Introduction.

Actually, {\bf D-FLAT\textasciicircum 2} uses the computation of preferred extensions as an example. When we followed the instruction provided by the authors\footnote{{\bf D-FLAT\textasciicircum 2} software and instructions have been retrieved from \url{https://github.com/hmarkus/dflat-2} in March 2016.}, we were able to compute the preferred extensions of small \AFname{}s, but the system was already struggling with resource consumption on medium sized instances. On the ICCMA2015 benchmarks, the system ran out of memory very quickly, and we did not obtain any solutions for any of the ICCMA2015 benchmarks. This is probably due to the fact that ICCMA instances have large tree-width and {\bf D-FLAT\textasciicircum 2} relies heavily on tree decomposition. For this reason, in the remainder of this section we focus our comparison on {\bf asprin}. 

As {\bf asprin} is based on ASP, the most natural comparison is against the ASP implementation of the proposed approach, namely \asp{}. For {\bf asprin}, we used clingo 4.5.2, the same version that is used as a backend for \asp{}.

As input to {\bf asprin} we use the program $\admaspenc{\anAFsymbol}$ in Definition~\ref{def:admasp} together with the following preference definition, which makes {\bf asprin} compute those answer sets that are subset maximal for atoms with the predicate $\mathtt{in}$.

  \begin{align*}
&\mathtt{\#preference(p1,superset)\{} \mathtt{in(X) : arg(X)}\mathtt{\}.}\\
&\mathtt{\#optimize(p1).}
\end{align*}

\begin{table}
\centering
\caption{\label{tab:asprin} Comparison of \asp{} and {\bf asprin}, on the ICCMA2015 benchmarks. Results are shown in terms of IPC score (maximum achievable is 192.0), percentages of success (\% Success) and PAR10.} 
\renewcommand{\arraystretch}{1.2}
\begin{tabular}{ l r r r }
\toprule

& IPC score & \% Success & PAR10 \\
 \midrule
\textbf{\asp{}} & 191.2 & 100.0 & 23.1 \\
\hdashline
\textbf{\bf asprin} & 157.8 & 100.0 & 44.9 \\
\bottomrule

\end{tabular}

\end{table}

The results of comparison between {\bf asprin} and \asp{} performed on the ICCMA2015 benchmarks for enumerating preferred extensions are shown in Table \ref{tab:asprin}. Results indicate that the proposed \asp{} system is significantly faster: \asp{} achieves an IPC score of 191.2 versus 157.8 of {\bf asprin}. According to the results, \asp{} is the fastest system on 187 of the considered \AFname{}s. This is also confirmed by the PAR10 scores; on average {\bf asprin} is about 20 seconds slower than \asp{}, while in terms of coverage, both the considered systems are able to successfully analyse all the 192 \AFname{}s of the benchmark set.   

\subsection{Comparison with Abstract Argumentation Algorithms Based on the Same Approaches}

According to the results of ICCMA2015 \cite{thimm2016summary}, {\bf ASPARTIX-V} \cite{thimm2015system} is the ASP-based abstract argumentation solver that showed the best performance in the preferred enumeration track. Table \ref{tab:asp-v} presents the results of a comparison between \asp{} and {\bf ASPARTIX-V} performed on the ICCMA2015 benchmarks. Presented results indicate that \asp{} is faster, both in terms of IPC and PAR10 scores. Remarkably, \asp{} is able to successfully analyse a larger number of \AFname{}s (100.0\% against 94.0\%). 

\begin{table}
\centering
\caption{\label{tab:asp-v} Comparison of \asp{} and {\bf ASPARTIX-V}, the ASP-based abstract argumentation solver that showed the best performance in the preferred enumeration track, on the ICCMA2015 benchmarks. Results are shown in terms of IPC score (maximum achievable is 192.0), percentages of success (\% Success) and PAR10.} 
\renewcommand{\arraystretch}{1.2}
\begin{tabular}{ l r r r }
\toprule

& IPC score & \% Success & PAR10 \\
 \midrule
\textbf{\asp{}} & 171.3 & 100.0 & 23.1 \\
\hdashline
\textbf{\bf ASPARTIX-V} & 148.5 & 94.0 & 630.9 \\
\bottomrule

\end{tabular}

\end{table}

At a closer look, it is noticeable that---among ICCMA2015 frameworks---the \AFname{}s with a very large grounded extension and many nodes in general\footnote{http://argumentationcompetition.org/2015/results.html} are very challenging for {\bf ASPARTIX-V}, while the proposed \asp{} solver is able to quickly and effectively analyse also such large frameworks.


\subsection{Comparison with the State of the Art Solver}

In this analysis we compare \asp{} and \maxsat{} with the winner of the the ICCMA2015 track on enumerating preferred extensions, \winner{} \cite{cegartix}. Table \ref{tab:iccma} shows the performance of considered solvers in terms of IPC score, percentage of successfully analysed \AFname{}s and PAR10. \asp{} performs significantly better than \maxsat{}. This is possibly due to the fact that each preferred extension results from an execution of the MaxSAT solver; and a final run is needed in order to demonstrate that no other extensions exist. Therefore, the number of MaxSAT calls is exactly the number of preferred extensions plus one. The generated MaxSAT formulas tend to be large on the considered benchmarks; therefore, the added overhead can be remarkable.

\begin{table}
\centering
\caption{\label{tab:iccma} Comparison of \maxsat{} and \asp{} with the winner of the track of ICCMA2015 on enumerating preferred extensions, \winner{}. Results are shown in terms of IPC score (maximum 192.0), percentages of success and PAR10.} 
\renewcommand{\arraystretch}{1.2}
\begin{tabular}{ l r r r }
\toprule

& IPC score & \% Success & PAR10 \\
 \midrule
\asp{} & 161.7 & 100.0 & 23.1 \\
\hdashline
\maxsat{} & 115.3 & 85.0 & 1423.8 \\
\hdashline
\winner{} & 188.9 & 100.0 & 15.2 \\
\bottomrule

\end{tabular}

\begin{tabular}{ l r r r}
\toprule


\end{tabular}

\end{table}

Interestingly, the performance of \asp{} is comparable to the performance of \winner{}; according to PAR10, \asp{} needs on average 8 seconds more to enumerate the preferred extensions. Moreover, by re-running the top solvers that took part in this track of ICCMA2015, we observed that \asp{} would have been ranked second. This is an impressive achievement, considering that the described algorithm: (i) is very general, in the sense that it does not exploit any argumentation-specific knowledge; (ii) is very easy to implement, particularly in the ASP configuration; and (iii) has been implemented as a prototype, without attention on software engineering techniques for improving performance.



\section{Conclusions and Future Work}

We have proposed a general methodology for solving subset optimality problems by means of iteratively solving cardinality optimality problems. This approach is motivated by the availability of efficient systems that support finding cardinality optimal solutions, namely MaxSAT solvers and ASP solvers supporting weak constraints.

As a methodology showcase we have produced the prototype system \asp{}, for enumerating preferred extensions of abstract argumentation frameworks. While the algorithms are general and easy to implement, an experimental analysis showed that they are competitive with the state-of-the-art system, which is specialized for this particular problem. On this showcase, our methods also prove higher performance than the existing general methods for computing subset optimal solutions of answer set programs, viz.~{\bf asprin} and {\bf D-FLAT\textasciicircum 2}.

\nop{
There are quite many opportunities for future work. As discussed in \cite{liffiton2008algorithms}, retain information between calls could help the performance, and we might investigate it in future work. However, this impact on the modularity of the concrete approach---i.e., a solver must be integrated in the framework---and can potentially reduce the generality of the overall framework. }

Apart from tuning the prototype implementation \asp{} to improve its performance, we intend to apply the methodology also to other application domains. Diagnosis or minimal model computation are immediate candidates. Another possibility is integrating our algorithm into a system like {\bf asprin}, or one that supports the same input language.

The methodology would also allow for computing $\SigmaP{3}$-hard problems when using ASP, which would be interesting to explore, as it would give rise to alternatives to implementations relying on QBFs.
It would also be worthwhile to explore whether the general methodology can be used also with formalisms different from MaxSAT and ASP, which would open entirely new avenues.

\section*{Acknowledgements}

The authors would like to acknowledge the use of the University
of Huddersfield Queensgate Grid in carrying out this
work. 

\bibliography{references}

\end{document}